\documentclass[12pt]{article}

\usepackage{color}
\usepackage[utf8]{inputenc}
\usepackage[T1]{fontenc}    
\usepackage{url}            
\usepackage{booktabs}       
\usepackage{amsfonts}       
\usepackage{nicefrac}       
\usepackage{microtype}      

\usepackage{amsmath}
\usepackage{amsfonts}
\usepackage{mathrsfs}
\usepackage{amssymb}
\usepackage{amsthm}
\usepackage{wrapfig}

\usepackage{thmtools}
\usepackage{thm-restate}

\usepackage{cleveref}

\usepackage{graphicx,subfigure}
\usepackage{tikz}
\usepackage{comment}

\usepackage{epstopdf}

\usepackage{libertine}

\usepackage{multirow}
\usepackage{wrapfig}

\usepackage{algorithm}
\usepackage{algorithmic}

\usepackage[numbers]{natbib}
\newdimen\arrowsize
\pgfarrowsdeclare{arcsq}{arcsq}
{
  \arrowsize=0.2pt
  \advance\arrowsize by .5\pgflinewidth
  \pgfarrowsleftextend{-4\arrowsize-.5\pgflinewidth}
  \pgfarrowsrightextend{.5\pgflinewidth}
}
{
  \arrowsize=1.5pt
  \advance\arrowsize by .5\pgflinewidth
  \pgfsetdash{}{0pt} 
  \pgfsetroundjoin   
  \pgfsetroundcap    
  \pgfpathmoveto{\pgfpoint{0\arrowsize}{0\arrowsize}}
  \pgfpatharc{-90}{-140}{4\arrowsize}
  \pgfusepathqstroke
  \pgfpathmoveto{\pgfpointorigin}
  \pgfpatharc{90}{140}{4\arrowsize}
  \pgfusepathqstroke
}

\newtheorem{mythm}{Theorem}
\newtheorem{mylm}{Lemma}

\newtheorem{myasum}{Assumption}
\newtheorem{mycor}{Corollary}

\usepackage{chngcntr}
\usepackage{apptools}
\AtAppendix{\counterwithin{mytheorem}{section}}

\AtAppendix{\counterwithin{mylemma}{section}}
\newtheorem{mylemma}{Lemma}
\AtAppendix{\counterwithin{equation}{section}}

\title{Learning with Biased Complementary Labels}

\font\myfont=cmr12 at 13pt

\author{\myfont Xiyu~Yu\thanks{UBTECH Sydney AI Centre and the School of Information Technologies in the Faculty Engineering and Information Technologies at The University of Sydney, NSW, 2006, Australia, xiyu0300@uni.sydney.edu.au, tongliang.liu@sydney.edu.au, dacheng.tao@sydney.edu.au} \ \ \   Tongliang~Liu\footnotemark[1] \ \ \   
Mingming Gong\thanks{Department of Biomedical Informatics, University of Pittsburgh,gongmingnju@gmail.com}~\thanks{Department of Philosophy, Carnegie Mellon University} \ \ \  Dacheng Tao\footnotemark[1]}

\date{}

\begin{document}
\bibliographystyle{plain}

\maketitle

\begin{abstract}
 In this paper, we study the classification problem in which we have access to easily obtainable surrogate for true labels, namely complementary labels, which specify classes that observations do \textbf{not} belong to. Let $Y$ and $\bar{Y}$ be the true and complementary labels, respectively. We first model the annotation of complementary labels via transition probabilities $P(\bar{Y}=i|Y=j), i\neq j\in\{1,\cdots,c\}$, where $c$ is the number of classes. Previous methods implicitly assume that $P(\bar{Y}=i|Y=j), \forall i\neq j$, are identical, which is not true in practice because humans are biased toward their own experience. For example, as shown in Figure \ref{complementary_label_cases}, if an annotator is more familiar with monkeys than prairie dogs when providing complementary labels for meerkats, she is more likely to employ ``monkey'' as a complementary label. We therefore reason that the transition probabilities will be different. In this paper, we propose a framework that contributes three main innovations to learning with \textbf{biased} complementary labels: (1) It estimates transition probabilities with no bias. (2) It provides a general method to modify traditional loss functions and extends standard deep neural network classifiers to learn with biased complementary labels. (3)  It theoretically ensures that the classifier learned with complementary labels converges to the optimal one learned with true labels. Comprehensive experiments on several benchmark datasets validate the superiority of our method to current state-of-the-art methods.
\end{abstract}

\newpage
\section{Introduction}
Large-scale training datasets translate supervised learning from theories and algorithms to practice, especially in deep supervised learning. One major assumption that guarantees this successful translation is that data are accurately labeled. However, collecting true labels for large-scale datasets is often expensive, time-consuming, and sometimes impossible. For this reason, some weak but cheap supervision information has been exploited to boost learning performance. Such supervision includes side information \cite{xing2003distance}, privileged information \cite{vapnik2009new}, and weakly supervised information \cite{Law_2017_CVPR} based on semi-supervised data \cite{zhu2005semi,haeusser2017learning,Abbasnejad_2017_CVPR}, positive and unlabeled data \cite{du2014analysis}, or noisy labeled data \cite{misra2016seeing,Veit_2017_CVPR,han2018progressive,han2018co,cheng2017learning,gong2017learning}.  In this paper, we study another weak supervision: the complementary label which specifies a class that an object does \textbf{not} belong to. Complementary labels are sometimes easily obtainable, especially when the class set is relatively large. Given an observation in multi-class classification, identifying a class label that is incorrect for the observation is often much easier than identifying the true label.

\begin{figure}[t]
\begin{center}
{\includegraphics[width=0.85\columnwidth]{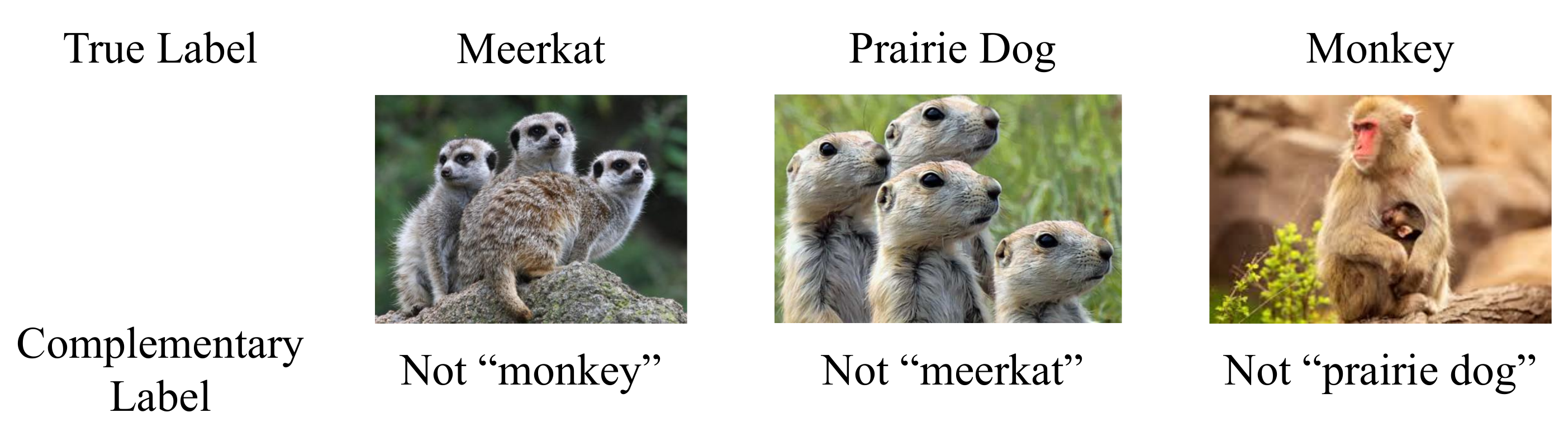}} 
\caption{A comparison between true labels (top) and complementary labels (bottom). } \label{complementary_label_cases}
\end{center}
\end{figure}

Complementary labels carry useful information and are widely used in our daily lives: for example, to identify a language we do not know, we may say ``not English''; to categorize a new movie without any fighting, we may say ``not action''; and to recognize an image of a previous American president, we may say ``not Trump''. Ishida et al. \cite{ishida2017learning} then proposed learning from examples with only complementary labels by assuming that a complementary label is uniformly selected from the $c-1$ classes other than the true label class ($c>2$). Specifically, they designed an unbiased estimator such that learning with complementary labels was asymptotically consistent with learning with true labels. 

Sometimes, annotators provide complementary labels based on both the content of observations and their own experience, leading to the biases in complementary labels. \begin{wrapfigure}{r}{0.31\columnwidth}
{\includegraphics[width=0.3\columnwidth]{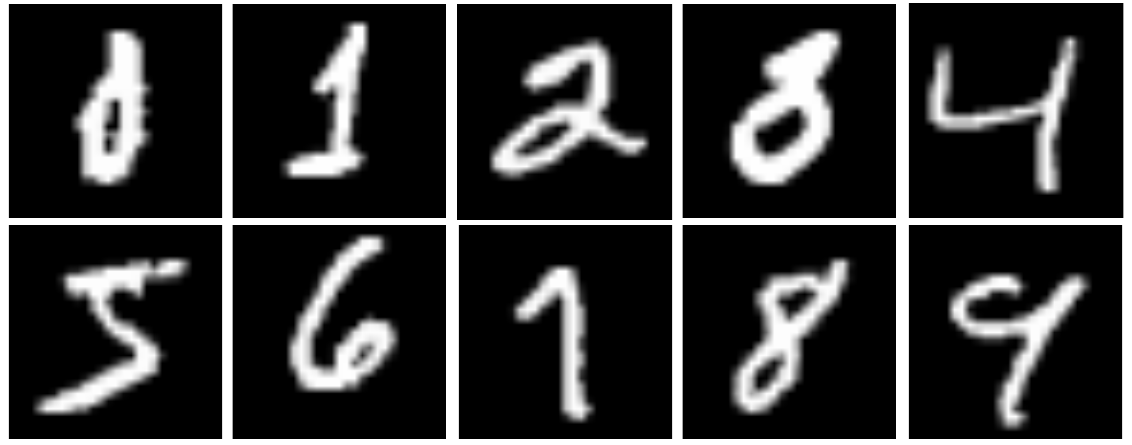}}
\end{wrapfigure} Thus, complementary labels are mostly non-uniformly selected from the remaining $c-1$ classes, some of which even have no chance of being selected for certain cases. Regarding the bias governed by the observation content, let us take labeling digits 0-9 as an example. Since digit 1 is much more dissimilar to digit 3 than digit 8, the complementary labels of ``3'' are more likely to be assigned with ``1'' rather than ``8''. Regarding the bias governed by annotators' experience, taking our example above, we can see that if one is more familiar with monkeys than other animals, she may be more likely to use ``monkey'' as a complementary label.

Motivated by the cause of biases, we here model the biased procedure of annotating complementary labels via probabilities $P(\bar{Y}=i|Y=j), i\neq j\in\{1,\cdots,c\}$. Note that the assumption that a complementary label is uniformly selected from the remaining $c-1$ classes implies $P(\bar{Y}=i|Y=j)=1/(c-1), i\neq j\in\{1,\cdots,c\}$. However, in real applications, the probabilities should not be $1/(c-1)$ and can differ vastly. How to estimate the probabilities is a key problem for learning with complementary labels.

We therefore address the problem of learning with biased complementary labels. For effective learning, we propose to estimate the probabilities $P(\bar{Y}=i|Y=j), i\neq j\in\{1,\cdots,c\}$ without biases. Specifically, we prove that given a clear observation $\mathbf{x}_j$ for the $j$-th class, i.e., the observation satisfying $P(Y=j|\mathbf{x}_j)=1$, which can be easily identified by the annotator, it holds that $P(\bar{Y}=i|Y=j)=P(\bar{Y}=i|\mathbf{x}_j), i\in\{1,\cdots,j-1,j,\cdots,c\}$. This implies that probabilities $P(\bar{Y}=i|Y=j), i\neq j\in\{1,\cdots,c\}$ can be estimated without biases by learning $P(\bar{Y}=i|\mathbf{x}_j)$ from the examples with complementary labels. To obtain these clear observations, we assume that a small set of easily distinguishable instances (e.g., 10 instances per class) is usually not expensive to obtain.

Given the probabilities $P(\bar{Y}|Y)$, we modify traditional loss functions proposed for learning with true labels so that the modifications can be employed to efficiently learn with biased complementary labels. We also prove that by exploiting examples with complementary labels, the learned classifier converges to the optimal one learned with true labels with a guaranteed rate. Moreover, we also empirically show that the convergence of our method benefits more from the biased setting than from the uniform assumption, meaning that we can use a small training sample to achieve a high performance.

Comprehensive experiments are conducted on benchmark datasets including UCI, MNIST, CIFAR, and Tiny ImageNet, which verifies that our method significantly outperforms the state-of-the-art methods with accuracy gains of over 10\%. We also compare the performance of classifiers learned with complementary labels to those learned with true labels. The results show that our method almost attains the performance of learning with true labels in some situations.

\section{Related Work}
\paragraph{Learning with complementary labels.} To the best of our knowledge, Ishida et al. \cite{ishida2017learning} is the first to study learning with complementary labels. They assumed that the transition probabilities are identical and then proposed modifying traditional one-versus-all (OVA) and pairwise-comparison (PC) losses for learning with complementary labels. The main differences between our method and \cite{ishida2017learning} are: (1) Our work is motivated by the fact that annotating complementary labels are often affected by human biases. Thus, we study a different setting in which transition probabilities are different. (2) In \cite{ishida2017learning}, modifying OVA and PC losses is naturally suitable for the uniform setting and provides an unbiased estimator for the expected risk of classification with true labels. In this paper, our method can be generalized to many losses such as cross-entropy loss and directly provides an unbiased estimator for the risk minimizer. Due to these differences, \cite{ishida2017learning} often achieves promising performance in the uniform setting while our method achieves good performance in both the uniform and non-uniform setting.

\paragraph{Learning with noisy labels.} In the setting of label noise, transition probabilities are introduced to statistically model the generation of noisy labels. In classification and transfer learning, methods \cite{natarajan2013learning,liu2016classification,wang2018multiclass,yu2017transfer} employ transition probabilities to modify loss functions such that they can be robust to noisy labels. Similar strategies to modify deep neural networks by adding a transition layer have been proposed in \cite{sukhbaatar2014training,patrini2016making}. However, this is the first time that this idea is applied to the new problem of learning with biased complementary labels. Different from label noise, here, all diagonal entries of the transition matrix are zeros and the transition matrix sometimes may be not required to be invertible in empirical.

\section{Problem Setup}
In multi-class classification, let $\mathcal{X} \in \mathbb{R}^d$ be the feature space and $\mathcal{Y}=[c]$ be the label space, where $d$ is the feature space dimension; $[c] = \{1,\cdots,c\}$; and $c>2$ is the number of classes. We assume that variables $(X,Y,\bar{Y})$ are defined on the space $\mathcal{X} \times \mathcal{Y} \times \mathcal{Y}$ with a joint probability measure $P(X,Y,\bar{Y})$ ($P_{XY\bar{Y}}$ for short). 

In practice, true labels are sometimes expensive but complementary labels are cheap. This work thus studies the setting in which we have a large set of training examples with biased complementary labels and a very small set of correctly labeled examples. The latter is only used for estimating transition probabilities. Our aim is to learn the optimal classifier with respect to the examples with true labels by exploiting the examples with complementary labels. 

For each example $(\mathbf{x},y)\in \mathcal{X} \times \mathcal{Y}$, a complementary label $\bar{y}$ is selected from the complement set $\mathcal{Y} \setminus \{y\}$. We assign a probability for each $\bar{y}\in \mathcal{Y} \setminus \{y\}$ to indicate how likely it can be selected, i.e., $P(\bar{Y}=\bar{y}|X=\mathbf{x},Y=y)$. In this paper, we assume that $\bar{Y}$ is independent of feature $X$ conditioned on true label $Y$, i.e., $P(\bar{Y}=\bar{y}|X=\mathbf{x},Y=y)=P(\bar{Y}=\bar{y}|Y=y)$. This assumption considers the bias which depends only on the classes, e.g., if the annotator is not familiar with the features in a specific class, she is likely to assign complementary labels that she is more familiar with. We summarize all the probabilities into a transition matrix $\mathbf{Q} \in \mathbb{R}^{c\times c}$, where $Q_{ij}=P(\bar{Y}=j|Y=i)$ and $Q_{ii}=0, \forall i,j\in[c]$. Here, $Q_{ij}$ denotes the entry value in the $i$-th row and $j$-th column of $\mathbf{Q}$. Note that transition matrix is also widely exploited in Markov chains \cite{gagniuc2017markov} and has many applications in machine learning, such as learning with label noise \cite{natarajan2013learning,sukhbaatar2014training,patrini2016making}. 

If complementary labels are uniformly selected from the complement set, then $\forall i,j\in [c] \textrm{ and } i\neq j$, $Q_{ij}=\frac{1}{c-1}$. Previous work \cite{ishida2017learning} has proven that the optimal classifier can be found under the uniform assumption. Sometimes, this is not true in practice due to human biases. Therefore, we focus on situations in which $Q_{ij}, \forall i \neq j$, are different. We mainly study the following problems: how to modify loss functions such that the classifier learned with these biased complementary labels can converge to the optimal one learned with true labels; the speed of the convergence; and how to estimate transition probabilities.

\section{Methodology}
In this section, we study how to learn with biased complementary labels. We first review how to learn optimal classifiers from examples with true labels. Then, we modify loss functions for complementary labels and propose a deep learning based model accordingly. Lastly, we theoretically prove that the classifier learned by our method is consistent with the optimal classifier learned with true labels.

\subsection{Learning with True Labels}  
The aim of multi-class classification is to learn a classifier $f(\mathbf{x})$ that predicts a label $y$ for a given observation $\mathbf{x}$. Typically, the classifier is of the following form:
\begin{equation} \label{classifier}
f(X)=\arg\max_{i\in[c]} g_i(X),
\end{equation}
where $\mathbf{g}:\mathcal{X} \to \mathbb{R}^c$ and $g_i(X)$ is the estimate of $P(Y=i|X)$. 

Various loss functions $\ell(f(X),Y)$ have been proposed to measure the risk of predicting $f(X)$ for $Y$ \cite{bartlett2006convexity}. Formally, the expected risk is defined as.
\begin{equation} \label{expected_risk_true}
R(f) = \mathbb{E}_{(X,Y)\sim P_{XY}}[\ell(f(X),Y)].
\end{equation}
The optimal classifier is the one that minimizes the expected risk; that is,
\begin{equation} \label{optimal}
f^* = \arg\min_{f \in \mathcal{F}} R(f),
\end{equation}
where $\mathcal{F}$ is the space of $f$.

However, the distribution $P_{XY}$ is usually unknown. We then approximate $R(f)$ by using its empirical counterpart: 
$R_n(f)=\frac{1}{n}\sum_{i=1}^n \ell(f(\mathbf{x}_i),y_i)$, where $\{(\mathbf{x}_i,y_i)\}_{1\leq i\leq n}$ are i.i.d. examples drawn according to $P_{XY}$.

Similarly, the optimal classifier is approximated by $f_n = \arg\min_{f \in \mathcal{F}}R_n(f)$.

\subsection{Learning with Complementary Labels} 
True labels, especially for large-scale datasets, are often laborious and expensive to obtain. We thus study an easily obtainable surrogate; that is, complementary labels. However, if we still use traditional loss functions $\ell$ when learning with these complementary labels, similar to Eq.(\ref{classifier}), we can only learn a mapping $\mathbf{q}:\mathcal X \to \mathbf{R}^c$ that tries to predict conditional probabilities $P(\bar{Y}|X)$ and the corresponding classifier that predicts a $\bar{y}$ for a given observation $\mathbf{x}$.

Therefore, we need to modify these loss functions such that the classifier learned with biased complementary labels can converge to the optimal one learned with true labels. Specifically, let $\bar{\ell}$ be the modified loss function. Then, the expected and empirical risks with respect to complementary labels are defined as $\bar{R}(f) = \mathbb{E}_{(X,\bar{Y})\sim P_{X\bar{Y}}}[\bar{\ell}(f(X),\bar{Y})]$ and $\bar{R}_n(f) = \frac{1}{n}\sum_{i=1}^n \bar{\ell}(f(\mathbf{x}_i),\bar{y}_i)]$,
respectively. Here, $\{(\mathbf{x}_i,\bar{y}_i)\}_{1\leq i\leq n}$ are examples with complementary labels.

Denote $\bar{f}^*$ and $\bar{f}_n$ as the optimal solution obtained by minimizing $\bar{R}(f)$ and $\bar{R}_n(f)$, respectively. They are $\bar{f}^* = \arg\min_{f\in \mathcal{F}} \bar{R}(f)$ and $\bar{f}_n = \arg\min_{f \in \mathcal{F}} \bar{R}_n(f)$.

We hope that the modified loss function $\bar{\ell}$ can ensure that $\bar{f}_n \stackrel{n}{\longrightarrow} f^*$, which implies that by learning with complementary labels, the classifier we obtain can also approach to the optimal one defined in (\ref{optimal}).

Recall that in transition matrix $\mathbf{Q}$, $Q_{ij} = P(\bar{Y}=j|Y=i)$ and $Q_{ii} = P(\bar{Y}=i|Y=i)=0, \forall i \in [c]$. We observe that $P(Y|X)$ can be transferred to $P(\bar{Y}|X)$ by using the transition matrix $\mathbf{Q}$; that is, $\forall j\in [c]$,
\begin{equation}\label{flip_process}
\begin{aligned}
P(\bar{Y}=j|X) &= \sum_{i\neq j} P(\bar{Y}=j,Y=i|X) \\
&= \sum_{i\neq j} P(\bar{Y}=j|Y=i,X) P(Y=i|X)\\
&= \sum_{i\neq j} P(\bar{Y}=j|Y=i) P(Y=i|X).
\end{aligned}
\end{equation}

\begin{figure}[t]
\begin{center}
{\includegraphics[width=0.85\columnwidth]{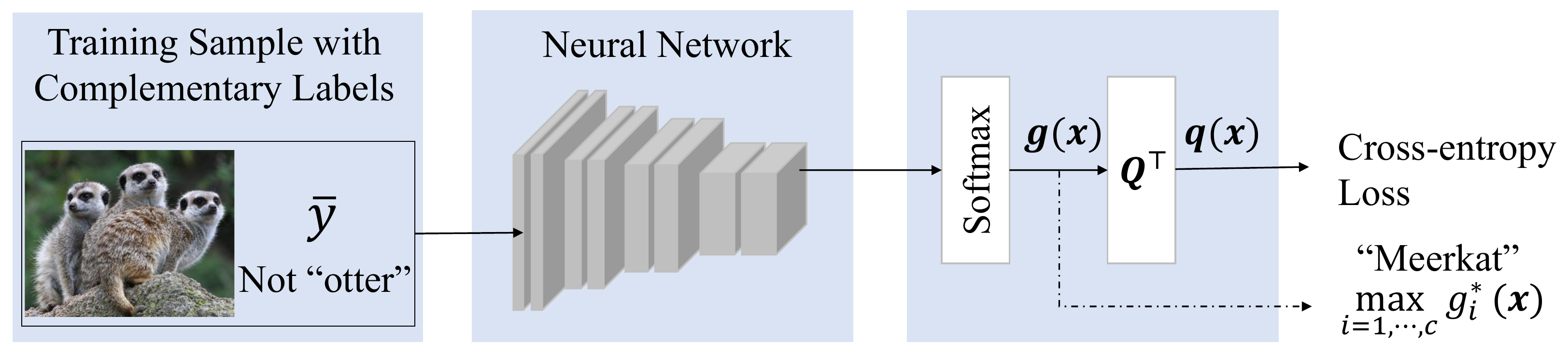}}
\caption{An overview of our method. We modify the deep neural network by adding a layer that multiplies the output of the softmax function by $\mathbf{Q}^\top$.} \label{fig_overview}
\end{center}
\end{figure}

Intuitively, if $q_i(X)$ tries to predict the probability $P(\bar{Y}=i|X)$, $\forall i \in [c]$, then $\mathbf{Q}^{-\top}\mathbf{q}$ can predict the probability $P(Y|X)$. To enable end-to-end learning rather than transferring after training, we let
\begin{equation}
\mathbf{q}(X) = \mathbf{Q}^{\top} \mathbf{g}(X),
\end{equation}
where $\mathbf{g}(X)$ is now an intermediate output, and $f(X)=\arg\max_{i\in [c]}g_i(X)$.

Then, the modified loss function $\bar{\ell}$ is
\begin{equation} \label{modified_loss0}
\bar{\ell}(f(X), \bar{Y}) = \ell(\mathbf{q}(X), \bar{Y}).
\end{equation}
In this way, if we can learn an optimal $\mathbf{q}^*$ such that $q_i^*(X)=P(\bar{Y}=i|X), \forall i \in [c]$, meanwhile, we can also find the optimal $\mathbf{g}^*$ and the classifier $f^*$.

This loss modification method can be easily applied to deep learning. As shown in Figure \ref{fig_overview}, we achieve this simply by adding a linear layer to the deep neural network. This layer outputs $\mathbf{q}(X)$ by multiplying the output of the softmax function (i.e., $\mathbf{g}(X)$) by the transposed transition matrix $\mathbf{Q}^\top$. With sufficient training examples with complementary labels, this deep neural network often simultaneously learns good classifiers for both $(X,\bar{Y})$ and $(X,Y)$.

Note that, in our modification, the forward process does not need to compute $\mathbf{Q}^{-\top}$. Even though the subsequent analysis for identification requires the transition matrix to be invertible, sometimes, we may have no such requirement in practice. We also show an example in our experiments that even with singular transition matrices, high classification performance can also be achieved if no column of $\mathbf{Q}$ is all-zero.

\section{Identification of the Optimal Classifier}
In this section, we aim to prove that the proposed loss modification method ensures the identifiability of the optimal classifier under a reasonable assumption:
\begin{myasum} \label{learnability}
By minimizing the expected risk $R(f)$, the optimal mapping $\mathbf{g}^*$ satisfies $g_i^*(X) = P(Y=i|X), \forall i \in [c]$.
\end{myasum}

Based on Assumption \ref{learnability}, we can prove that $\bar{f}^*=f^*$ by the following theorem:
\begin{mythm}\label{thm_main}
Suppose that $\mathbf{Q}$ is invertible and Assumption \ref{learnability} is satisfied, then the minimizer $\bar{f}^*$ of $\bar{R}(f)$ is also the minimizer $f^*$ of $R(f)$; that is, $\bar{f}^*=f^*$.
\end{mythm}

Please find the detailed proof in Appendix A. 
Given sufficient training data with complementary labels, $\bar{f}_n$ can converge to $\bar{f}^*$, which can be proved in the next section. According to Theorem \ref{thm_main}, this also implies that $\bar{f}_n$ also converges to the optimal classifier $f^*$.


\paragraph{Examples of Loss Functions.}
The proof of Theorem \ref{thm_main} relies on Assumption \ref{learnability}. However, for many loss functions, Assumption \ref{learnability} can be provably satisfied. Here, we take the cross-entropy loss as an example to demonstrate this fact. The cross-entropy loss is widely used in deep supervised learning and is defined as
\begin{equation} \label{modified_loss}
\ell(f(X),Y) = -\sum_{i=1}^c 1(Y=i)\log(g_i(X)),
\end{equation}
where $1(\cdot)$ is an indicator function; that is, if the input statement is true, it outputs 1; otherwise, 0. For the cross-entropy loss, we have the following lemma:
\begin{mylm} \label{lm_prob}
Suppose $\ell$ is the cross-entropy loss and $\mathbf{g}(X) \in \Delta^{c-1}$, where $\Delta^{c-1}$ refers to a standard simplex in $\mathbb{R}^c$; that is, $\forall \mathbf{x} \in \Delta^{c-1}$, $x_i \geq 0, \forall i \in [c]$ and $\sum_{i=1}^c x_i = 1$. By minimizing the expected risk $R(f)$, we have $g_i^*(X) = P(Y=i|X), \forall i \in [c]$.
\end{mylm}

Please see the detailed proof in Appendix B. In fact, losses such as square-error loss $\ell(f(X),Y)=\sum_{j=1}^c (1(Y=j)-g_j(X))^2$, also satisfy Assumption 1. The readers can prove it themselves using similar strategy. Combined with Theorem \ref{thm_main}, we can see, by applying the proposed method to loss functions such as cross-entropy loss, we can prove that the optimal classifier $f^*$ can be found even when learning with biased complementary labels.

\section{Convergence Analysis}
In this section, we show an upper bound for the estimation error of our method. This upper bound illustrates a convergence rate for the classifier learned with complementary labels to the optimal one learned with true labels. Moreover, with the derived bound, we can clearly see that the estimation error could further benefit from the setting of biased complementary labels under mild conditions.

Since $\bar{f}^*=f^*$, we have $|\bar{f}_n-f^*|=|\bar{f}_n-\bar{f}^*|$. We will upper bound the error $|\bar{f}_n-\bar{f}^*|$ via upper bounding $\bar{R}(\bar{f}_n)-\bar{R}(\bar{f}^*)$; that is, when $\bar{R}(\bar{f}_n)-\bar{R}(\bar{f}^*)\to 0$, $|\bar{f}_n-\bar{f}^*| \to 0$. Specifically, it has been proven that
\begin{equation}
\begin{aligned}
\bar{R}(\bar{f}_n)-\bar{R}(\bar{f}^*) &= \bar{R}(\bar{f}_n)-\bar{R}_n(\bar{f}_n)+\bar{R}_n(\bar{f}_n)-\bar{R}_n(\bar{f}^*)+\bar{R}_n(\bar{f}^*)-\bar{R}(\bar{f}^*)\\
&\leq \bar{R}(\bar{f}_n)-\bar{R}_n(\bar{f}_n)+\bar{R}_n(\bar{f}^*)-\bar{R}(\bar{f}^*)\\
&\leq 2\sup_{f\in \mathcal{F}}|\bar{R}(f)-\bar{R}_n(f)|,
\end{aligned}
\end{equation}
where the first inequality holds because $\bar{R}_n(\bar{f}_n)-\bar{R}_n(\bar{f}^*)\leq 0$ and the error in the last line is called the generalization error.

Let $(X_1,\bar{Y}_1),\cdots,(X_n,\bar{Y}_n)$ be independent variables. By employing the concentration inequality \cite{boucheron2013concentration}, the generalization error can be upper bounded by using the method of Rademacher complexity \cite{bartlett2002rademacher}.
\begin{mythm}[\cite{bartlett2002rademacher}]\label{rademacher}
Let the loss function be upper bounded by $M$. Then, for any $\delta>0$, with the probability $1-\delta$, we have
\begin{equation}
\begin{aligned}
\sup_{f\in \mathcal{F}}|\bar{R}(f)-\bar{R}_n(f)|\leq 2\mathfrak{R}_n(\bar{\ell}\circ\mathcal{F})+M\sqrt{\frac{\log{1/\delta}}{2n}},
\end{aligned}
\end{equation}
where $\mathfrak{R}_n(\bar{\ell}\circ\mathcal{F})=\mathbb{E}\left[\sup_{f\in \mathcal{F}}\frac{1}{n}\sum_{i=1}^n\sigma_i\bar{\ell}(f(X_i),\bar{Y}_i)\right]$ is the Rademacher complexity; $\{\sigma_1,\cdots,\sigma_n\}$ are Rademacher variables uniformly distributed from $\{-1,1\}$.
\end{mythm}

Before upper bounding $\mathfrak{R}_n(\bar{\ell}\circ\mathcal{F})$, we need to discuss the specific form of the employed loss function $\bar{\ell}$. By exploiting the well-defined binary loss functions, one-versus-all and pairwise-comparison loss functions \cite{zhang2004statistical} have been proposed for multi-class learning.
In this section, we discuss the modified loss function $\bar{\ell}$ defined by Eqs. (\ref{modified_loss0}) and (\ref{modified_loss}), which can be rewritten as,
\begin{equation}\label{loss}
\begin{aligned}
\bar{\ell}(f(X),\bar{Y}) &= -\sum_{i=1}^c 1(\bar{Y}=i)\log\left((\mathbf{Q}^\top \mathbf{g})_i(X)\right) \\
&= -\sum_{i=1}^c 1(\bar{Y}=i)\log\left(\frac{\sum_{j=1}^c Q_{ji} \exp(h_j(X))}{\sum_{k=1}^c \exp(h_k(X))} \right),
\end{aligned}
\end{equation}
where $(\mathbf{Q}^\top \mathbf{g})_i$ denotes the $i$-th entry of $\mathbf{Q}^\top \mathbf{g}$; $\mathbf{h}:\mathcal{X}\to  \mathbb{R}^c$, $h_i(X) \in \mathcal{H}, \forall i \in [c]$; and $g_i(X)=\frac{\exp(h_i(X))}{\sum_{k=1}^c \exp(h_k(X))}$.

Usually, the convergence rates of generalization bounds of multi-class learning are at most $O(c^2/\sqrt{n})$ with respect to $c$ and $n$ \cite{ishida2017learning,mohri2012foundations}. To reduce the dependence on $c$ of our derived convergence rate, we rewrite $\bar{R}(f)$ as follows:
\begin{equation}
\begin{aligned}
&\bar{R}(f) =\int_{X} \sum_{i=1}^c P(\bar{Y}=i) P(X|\bar{Y}=i) \bar{\ell}(f(X),\bar{Y}=i) dX \\
&=\sum_{i=1}^c P(\bar{Y}=i) \int_{X} P(X|\bar{Y}=i)  \bar{\ell}(f(X),\bar{Y}=i) dX \\
&=\sum_{i=1}^c \bar{\pi}_i \bar{R}_{i}(f),
\end{aligned}
\end{equation}
where $\bar{R}_{i}( f) = \mathbb{E}_{X\sim P(X|\bar{Y}=i)} \bar{\ell}(f(X),\bar{Y}=i)$ and $\bar{\pi}_i=P(\bar{Y}=i)$. 

Similar to Theorem \ref{rademacher}, we have the following theorem.
\begin{mythm}\label{our_rademacher}
Suppose $\bar{\pi}_i=P(\bar{Y}=i)$ is given. Let the loss function be upper bounded by $M$. Then, for any $\delta>0$, with the probability $1-c\delta$, we have
\begin{equation}
\begin{aligned}
\bar{R}(\bar{f}_n)-\bar{R}(\bar{f}^*) &\leq 2\sup_{f\in \mathcal{F}}|\bar{R}(f)-\bar{R}_n(f)| \\
&\leq2\sum_{i=1}^c \bar{\pi}_i\sup_{f\in \mathcal{F}} |\bar{R}_{i}(f)-\bar{R}_{i,n_i}(f)|\\
&\leq 2\sum_{i=1}^c \bar{\pi}_i\left(2\mathfrak{R}_{n_{i}}( \bar{\ell}\circ\mathcal{F})+M\sqrt{\frac{\log{1/\delta}}{2n_{i}}}\right)\\
&= \sum_{i=1}^c \left( 4 \bar{\pi}_i \mathfrak{R}_{n_{i}}( \bar{\ell}\circ\mathcal{F}) + 2\bar{\pi}_iM\sqrt{\frac{\log{1/\delta}}{2n_{i}}} \right),
\end{aligned}
\end{equation}
where $\mathfrak{R}_{n_{i}}(\bar{\ell}\circ\mathcal{F})=\mathbb{E}\left[\sup_{f\in \mathcal{F}}\frac{1}{n_i}\sum_{j=1}^{n_i}\sigma_j\bar{\ell}(f(X_j),\bar{Y}_j=i)\right]$ and
$\bar{R}_{i,n_i}(f)$ is the empirical counterpart of $\bar{R}_i(f)$, and $n_{i},i\in [c]$, represents the numbers of $X$ whose complementary labels are $\bar{Y}=i$.
\end{mythm}

Due to the fact that $\bar{\ell}$ is actually defined with respect to $\mathbf{h}$ rather than $f$, we would like to bound the error by the Rademacher complexity of $\mathcal{H}$. We observe that the relationship between $\mathfrak{R}_{n_{i}}(\bar{\ell}\circ\mathcal{F})$ and $\mathfrak{R}_{n_{i}}(\mathcal{H})$ is:
\begin{mylm} \label{new_rademarcher_comp}
Let $\bar{\ell}(f(X),\bar{Y}=i) = -\log\left(\frac{\sum_{k=1}^c Q_{ki} \exp(h_k(X))}{\sum_{k=1}^c \exp(h_k(X))} \right)$ and suppose that $h_i(X) \in \mathcal{H}, \forall i \in [c]$, we have $\mathfrak{R}_{n_{i}}(\bar{\ell}\circ\mathcal{F}) \leq c \mathfrak{R}_{n_{i}}(\mathcal{H})$.
\end{mylm}

The detailed proof can be found in Appendix C. Combine Theorem \ref{our_rademacher} and Lemma \ref{new_rademarcher_comp}, we have the final result:
\begin{mycor}
Suppose $\bar{\pi}_i=P(\bar{Y}=i)$ is given. Let the loss function be upper bounded by $M$. Then, for any $\delta>0$, with the probability $1-c\delta$, we have
\begin{equation}
\bar{R}(\bar{f}_n)-\bar{R}(\bar{f}^*) \leq \sum_{i=1}^c \left(4c \bar{\pi}_i \mathfrak{R}_{n_{i}}(\mathcal{H}) + 2\bar{\pi}_i M\sqrt{\frac{\log{1/\delta}}{2n_{i}}}\right).
\end{equation}
\end{mycor}

In current state-of-the-art methods \cite{ishida2017learning}, the convergence rate of $\mathfrak{R}_n(\bar{\ell}\circ\mathcal{F})$ is of order $O(c^2/\sqrt{n})$ with respect to $c$ and $n$ while our derived bound $\sum_{i}^c 4c \bar{\pi}_i \mathfrak{R}_{n_{i}}(\mathcal{H})$ is of order $max_{i\in[c]}O(c/\sqrt{n_i})$. Since our error bound depends on $n_i$, the bound would be loose if $n_i$ (or $\bar{\pi}_i$) is small. However, if $\bar{\pi}_i$ is balanced and $n_i$ is about $n/c$, our convergence rate is of order $O(c\sqrt{c}/ \sqrt{n})$, which is smaller than the error bounds provided by previous methods if $c$ is very large. 

$\textbf{Remark.}$ Theorem 3 and Corollary 1 aim to provide the proof of uniform convergence for general losses and show how the convergence rate can benefit from the biased setting under mild conditions. Thus, assuming the loss is upper-bounded is reasonable for many loss functions such as the square-error loss. If the readers would like to derive specific error bound for the cross-entropy loss, strategies in \cite{wan2013regularization} can be employed. If we assume that the transition matrix $Q$ is invertible, we can derive similar results as those in Lemma 1-3 \cite{wan2013regularization} for the modified loss function, which can be finally deployed to derive generalization error bound similar to Corollary 1.


\section{Estimating $\mathbf{Q}$}
In the aforementioned method, transition matrix $\mathbf{Q}$ is assumed to be known, which is not true. Here, we thus provide an efficient method to estimate $\mathbf{Q}$.

When learning with complementary labels, we completely lose the information of true labels. Without any auxiliary information, it is impossible to estimate the transition matrix which is associated with the class priors of true labels. On the other hand, although it is costly to annotate a very large-scale dataset, a small set of easily distinguishable observations are assumed to be available in practice. This assumption is also widely used in estimating transition probabilities in label noise problem \cite{vahdat2017toward} and class priors in semi-supervised learning \cite{yu2018efficient}. Therefore, in order to estimate $\mathbf{Q}$, we manually assign true labels to 5 or 10 observations in each class. Since these selected observations are often easy to classify, we further assume that they satisfy the anchor set condition \cite{liu2016classification}:

\begin{myasum} [Anchor Set Condition]
For each class $y$, there exists an anchor set $\mathcal{S}_{\mathbf{x}|y} \subset \mathcal{X}$ such that $P(Y=y|X=\mathbf{x})=1$ and $P(Y=y'|X=\mathbf{x})=0$, $\forall y'\in \mathcal{Y}\setminus \{y\}, \mathbf{x} \in \mathcal{S}_{\mathbf{x}|y}$.
\end{myasum}
Here, $\mathcal{S}_{\mathbf{x}|y}$ is a subset of features in class $y$. Given several observations in $\mathcal{S}_{\mathbf{x}|y}, y\in [c]$, we are ready to estimate the transition matrix $\mathbf{Q}$. According to Eq. (\ref{flip_process}),
\begin{equation}
P(\bar{Y}=\bar{y}|X) = \sum_{y'\neq \bar{y}} P(\bar{Y}=\bar{y}|Y=y') P(Y=y'|X).
\end{equation}

Suppose $\mathbf{x} \in \mathcal{S}_{\mathbf{x}|y}$, then $P(Y=y|X=\mathbf{x})=1$ and $P(Y=y'|X=\mathbf{x})=0, \forall y'\in \mathcal{Y}\setminus \{y\}$. We have
\begin{equation} \label{estimate_pyy}
P(\bar{Y}=\bar{y}|X=\mathbf{x}) = P(\bar{Y}=\bar{y}|Y=y).
\end{equation}

That is, the probabilities in $\mathbf{Q}$ can be obtained via $P(\bar{Y}|X)$ given the observations in the anchor set of each class. Thus, we need only to estimate this conditional probability, which has been proven to be achievable in Lemma \ref{lm_prob}. In this paper, with the training sample $\{(x_i,\bar{y}_i)\}_{1\leq i\leq n}$, we estimate $P(\bar{Y}|X)$ by training a deep neural network with the softmax function and cross-entropy loss. After obtaining these conditional probabilities, each probability $P(\bar{Y}=\bar{y}|Y=y)$ in the transition matrix can be estimated by averaging the conditional probabilities $P(\bar{Y}=\bar{y}|X=\mathbf{x})$ on the anchor data $\mathbf{x}$ in class $y$. 

\section{Experiments}
We evaluate our algorithm on several benchmark datasets including the UCI datasets, USPS, MNIST \cite{mnistdigits}, CIFAR10, CIFAR100 \cite{krizhevsky2009learning}, and Tiny ImageNet\footnote{The dataset is available at \url{http://cs231n.stanford.edu/tiny-imagenet-200.zip}}. All our experiments are trained on neural networks. For USPS and UCI datasets, we employ a one-hidden-layer neural network ($d$-3-$c$) \cite{ishida2017learning}. For MNIST, LeNet-5 \cite{lecun1998gradient} is deployed, and ResNet \cite{he2016deep} is exploited for the other datasets. All models are implemented in PyTorch\footnote{\url{http://pytorch.org}}.


\subsection{UCI and USPS}
We first evaluate our method on USPS and six UCI datasets: WAVEFORM1, WAVEFORM2, SATIMAGE, PENDIGITS, DRIVE, and LETTER, downloaded from the UCI machine learning repository. We apply the same strategies of annotating complementary labels, standardization, validation, and optimization with those in \cite{ishida2017learning}. The learning rate is chosen from $\{10^{-5}, \cdots, 10^{-1}\}$, weight decay from $\{10^{-7},10^{-4},10^{-1}\}$, batch size 100.

For fair comparison in these experiments, we assume the transition probabilities are identical and known as prior. Thus, no examples with true labels are required here. All results are shown in Table \ref{tab_uci}. Our loss modification (“LM”) method is compared to a partial label (PL) method \cite{cour2011learning}, a multi-label (ML) method \cite{read2011classifier}, and ``PC/S'' (the pairwise-comparison formulation with sigmoid loss), which achieved the best performance in [12]. We can see, ``PC/S'' achieves very good performances. The relatively higher performance of our method may be due to that our method provides an unbiased estimator for risk minimizer.

\begin{table}[t!]
\centering
\caption{Classification accuracy on USPS and UCI datasets: the means and standard deviations of classification accuracy over 20 trials in percentages are reported. ``\#train'' is the number of training and validation examples in each class. ``\#test'' is the number of test examples in each class.}
\label{tab_uci}
\small
\begin{tabular}{c|c|c|c|c|c|c|c|c}
\hline
Dataset & $c$ & $d$ & \#train & \#test & PC/S & PL & ML & LM (ours) \\ \hline
WAVEFORM1 & $1 \sim 3$ & 21 & 1226 & 398 & \textbf{85.8 (0.5)} & \textbf{85.7 (0.9)} & 79.3 (4.8) & \textbf{85.1 (0.6)} \\ \hline
WAVEFORM2 & $1 \sim 3$ & 40 & 1227 & 408 & \textbf{84.7 (1.3)} & \textbf{84.6 (0.8)} & 74.9 (5.2) & \textbf{85.5 (1.1)} \\ \hline
SATIMAGE & $1 \sim 7$ & 36 & 415 & 211 & \textbf{68.7 (5.4)} & 60.7 (3.7) & 33.6 (6.2) & \textbf{69.3 (3.6)} \\ \hline
\multirow{5}{*}{PENDIGITS} & $1 \sim 5$ & \multirow{5}{*}{16} & 719 & 336 & \textbf{87.0 (2.9)} & 76.2 (3.3) & 44.7 (9.6) & \textbf{92.7 (3.7)} \\
                           &  $6 \sim 10$ & & 719 & 335 & 78.4 (4.6) & 71.1 (3.3) & 38.4 (9.6) & \textbf{85.8 (1.3)} \\
                           & even \# & & 719 & 336 & \textbf{90.8} (2.4) & 76.8 (1.6) & 43.8 (5.1) & \textbf{90.0 (1.0)} \\
                           & odd \# & & 719 & 335 & 76.0 (5.4) & 67.4 (2.6) & 40.2 (8.0) & \textbf{86.5 (0.5)} \\
                           & $1 \sim 10$ & & 719 & 335 & 38.0 (4.3) & 33.2 (3.8) & 16.1 (4.6) & \textbf{62.8 (5.6)} \\ \hline
\multirow{5}{*}{DRIVE} & $1 \sim 5$ & \multirow{5}{*}{48} & 3955 & 1326 & \textbf{89.1 (4.0)} & 77.7 (1.5) & 31.1 (3.5) & \textbf{93.3 (4.6)} \\
                           &  $6 \sim 10$ & & 3923 & 1313 & 88.8 (1.8) & 78.5 (2.6) & 30.4 (7.2) & \textbf{92.8 (0.9)} \\
                           & even \# & & 3925 & 1283 & \textbf{81.8 (3.4)} & 63.9 (1.8) & 29.7 (6.3) & \textbf{84.3 (0.7)} \\
                           & odd \# & & 3939 & 1278 & \textbf{85.4 (4.2)} & 74.9 (3.2) & 27.6 (5.8) & \textbf{85.9 (2.1)} \\
                           & $1 \sim 10$ & & 3925 & 1269 & 40.8 (4.3) & 32.0 (4.1) & 12.7 (3.1) & \textbf{75.1 (3.2)} \\ \hline 
\multirow{6}{*}{LETTER} & $1 \sim 5$ & \multirow{6}{*}{16} & 565 & 171 & \textbf{79.7 (5.4)} & 75.1 (4.4) & 28.3 (10.4) & \textbf{84.3 (1.5)} \\
                           &  $6 \sim 10$ & & 550 & 178 & 76.2 (6.2) & 66.8 (2.5) & 34.0 (6.9) & \textbf{84.4 (1.0)} \\
                           & $11 \sim 15$ & & 556 & 177 & 78.3 (4.1) & 67.4 (3.4) & 28.6 (5.0) & \textbf{88.3 (1.9)} \\
                           & $16 \sim 20$ & & 550 & 184 & 77.2 (3.2) & 68.4 (2.1) & 32.7 (6.4) & \textbf{85.2 (0.7)} \\
                           & $21 \sim 25$ & & 585 & 167 & 80.4 (4.2) & 75.1 (1.9) & 32.0 (5.7) & \textbf{82.5 (1.0)} \\
                           & $1 \sim 25$ & & 550 & 167 & \textbf{5.1 (2.1)} & \textbf{5.0 (1.0)} & \textbf{5.2 (1.1)} & \textbf{7.0 (3.6)} \\ \hline
\multirow{5}{*}{USPS} & $1 \sim 5$ & \multirow{5}{*}{256} & 652 & 166 & \textbf{79.1 (3.1)} & 70.3 (3.2) & 44.4 (8.9) & \textbf{86.4 (4.5)} \\
                           &  $6 \sim 10$ & & 542 & 147 & 69.5 (6.5) & 66.1 (2.4) & 37.3 (8.8) & \textbf{88.1 (2.7)} \\
                           & even \# & & 556 & 147 & 67.4 (5.4) & 66.2 (2.3) & 35.7 (6.6) & \textbf{79.5 (5.4)} \\
                           & odd \# & & 542 & 147 & 77.5 (4.5) & 69.3 (3.1) & 36.6 (7.5) & \textbf{86.3 (3.1)} \\
                           & $1 \sim 10$ & & 542 & 127 & \textbf{30.7 (4.4)} & 26.0 (3.5) & 13.3 (5.4) & \textbf{37.2 (5.4)} \\ \hline                
\end{tabular}
\end{table}

\subsection{MNIST}
MNIST is a handwritten digit dataset including 60,000 training images and 10,000 test images from 10 classes. To evaluate the effectiveness of our method, we consider the following three settings: (1) for each image in class $y$, the complementary label is uniformly selected from $\mathcal{Y}\setminus \{y\}$ (``\textbf{uniform}''); (2) the complementary label is non-uniformly selected, but each label in $\mathcal{Y}\setminus \{y\}$ has non-zero probability to be selected (``\textbf{without0}''); (3) the complementary label is non-uniformly selected from a small subset of $\mathcal{Y}\setminus \{y\}$ (``\textbf{with0}''). 

To generate complementary labels, we first give the probability of each complementary label to be selected. In the ``uniform'' setting, $P(\bar{Y}=j|Y=i)=\frac{1}{9}, \forall i \neq j$. In the ``without0'' setting, for each class $y$, we first randomly split $\mathcal{Y}\setminus \{y\}$ to three subsets, each containing three elements. Then, for each complementary label in these three subsets, the probabilities are set to $\frac{0.6}{3}$, $\frac{0.3}{3}$, and $\frac{0.1}{3}$, respectively. In the ``with0'' setting, for each class $y$, we first randomly selected three labels in $\mathcal{Y}\setminus \{y\}$, and then randomly assign them with three probabilities whose summation is 1. After $\mathbf{Q}$ is given, we assign complementary label to each image based on these probabilities. Finally, we randomly set aside 10\% of training data as validation set.

In all experiments, the learning rate is fixed to $1e-4$; batch size 128; weight decay $1e-4$; maximum iterations 60,000; and stochastic gradient descend (SGD) with momentum $\gamma=0.9$ \cite{sutskever2013importance} is applied to optimize deep models.  Note that, as shown in \cite{ishida2017learning} and previous experiments, \cite{ishida2017learning} and our method have surpassed baseline methods such as PL and ML. In the following experiments, we will not again make comparisons with these baselines.

The results are shown in Table \ref{tab_mnist}. The means and standard deviations of classification accuracy over five trials are reported. Note that the digit data features are not too entangled, making it easier to learn a good classifier. However, we can still see the differences in the performance caused by the change of settings for annotating complementary labels. According to the results shown in Table \ref{tab_mnist}, ``PC/S'' \cite{ishida2017learning} works relatively well under the uniform assumption but the accuracy deteriorates in other settings. Our method performs well in all settings. It can also be seen that due to the accurate estimates of these probabilities, ``LM/E'' with the estimated transition matrix $\mathbf{Q}$ is competitive with ``LM/T'' which exploits the true one.

\begin{table}[t!]
\centering
\caption{Classification accuracy on MNIST: the means and standard deviations of classification accuracy over five trials in percentages are reported. ``TL'' denotes the result of learning with true labels. ``LM/T'' and ``LM/E'' refer to our method with the true $\mathbf{Q}$ and the estimated one, respectively.}
\label{tab_mnist}
\begin{tabular}{l|c|c|c}
\hline
Method & Uniform & Without0 & With0 \\ \hline
TL & 99.12 & 99.12 & 99.12 \\
PC/S & $86.59\pm3.99$ & $76.03\pm3.34$  & $29.12\pm1.94$ \\
LM/T & $97.18\pm0.45$ & $97.65\pm0.15$ & $98.63\pm0.05$ \\
LM/E & $96.33\pm0.31$ &  $97.04\pm0.31$ & $98.61\pm0.05$ \\
\hline
\end{tabular}
\end{table}

\subsection{CIFAR10}
\begin{table}[t!]
\centering
\caption{Classification accuracy on CIFAR10: the means and standard deviations of classification accuracy over five trials in percentages are reported. ``TL'' denotes the result of learning with true labels. ``LM/T'' and ``LM/E'' refer to our method with the true $\mathbf{Q}$ and the estimated one, respectively.}
\label{tab_cifar10}
\begin{tabular}{l|c|c|c}
\hline
Method & Uniform & Without0 & With0 \\ \hline
TL & 90.78 & 90.78 & 90.78 \\
PC/S & $41.19\pm0.04$ & $42.97\pm3.00$  & $18.12\pm1.45$ \\
LM/T & $73.38\pm1.06$ & $78.80\pm0.45$ & $85.32\pm1.11$ \\
LM/E & $42.96\pm0.76$ &  $70.56\pm0.34$ & $84.60\pm0.14$ \\
\hline
\end{tabular}
\end{table}
We evaluate our method on the CIFAR10 dataset under the aforementioned three settings. CIFAR10 has totally 10 classes of tiny images, which includes 50,000 training images and 10,000 test images. We leave out 10\% of the training data as validation set. In these experiments, ResNet-18 \cite{he2016deep} is deployed. We start with an initial learning rate 0.01 and divide it by 10 after 40 and 80 epochs. The weight decay is set to $5e-4$, and other settings are the same as those for MNIST. Early stopping is applied to avoid overfitting. 

We apply the same process as MNIST to generate complementary labels. The results in Table \ref{tab_cifar10} verify the effectiveness of our method. ``PC/S'' achieves promising performance when complementary labels are uniformly selected, and our method outperforms ``PC/S'' in other settings.  In the ``uniform'' setting, $P(\bar{Y}|X)$ is not well estimated. As a result, the transition matrix is also poorly estimated. ``LM/E'' thus performs relatively badly.

The results of our method under the ``uniform'' and ``without0'' settings (shown in Table \ref{tab_cifar10}) are usually worse than that of ``with0''. For a certain amount of training images, the empirical results show that in the``uniform'' and ``without0'' setting, the proposed method converges at a slower rate than in the ``with0'' setting. This phenomenon may be caused by the fact that the uncertainty involved with the transition procedure in the ``with0'' setting is less than that in ``uniform'' and ``without0'' settings, making it easier to learn in the former setting. This phenomenon also indicates that, for images in each class, annotators need not to assign all possible complementary labels, but can provide the labels following the criteria, i.e., each label in the label space should be assigned as complementary label for images in at least one class. In this way, we can reduce the number of training examples to achieve high performance. 

\begin{table}[t!]
\centering
\caption{Classification accuracy on CIFAR100 and Tiny ImageNet under the setting ``with0'': the means and standard deviations of classification accuracy over five trials in percentages are reported. ``TL'' denotes the result of learning with true labels. ``LM/T'' and ``LM/E'' refer to our method with the true $\mathbf{Q}$ and the estimated one, respectively.}
\label{tab_cifar100}
\begin{tabular}{l|c|c}
\hline
Method & CIFAR100 & Tiny ImageNet\\ \hline
TL & 69.55 & 63.26 \\
PC/S & $8.95\pm1.47$ & N/A\\
LM/T & $62.84\pm0.30$ & $52.71 \pm 0.71$\\
LM/E & $60.27\pm0.28$ & $49.70 \pm 0.78$\\
\hline
\end{tabular}
\end{table}

\subsection{CIFAR100} 
CIFAR100 also presents a collection of tiny images including 50,000 training images and 10,000 test images. But CIFAR100 has totally 100 classes, each with only 500 training images. Due to the label space being very large and the number of training data being limited, in both ``uniform'' and ``without0'' settings, few training data are assigned as $j$ for images in each class $i$, $\forall i\neq j$. Both the proposed method and ``PC/S'' cannot converge. Here, we only conduct the experiments under the  ``with0'' setting. To generate complementary labels, for each class $y$, we randomly selected 5 labels from $\mathcal{Y} \setminus \{y\}$, and assign them with non-zero probabilities. Others have no chance to be selected.

In these experiments, ResNet-34 is deployed. Other experimental settings are the same with those in CIFAR10. Results are shown in the second column of Table \ref{tab_cifar100}. ``PC/S'' can hardly obtains a good classifier, but our method achieves high accuracies that are comparable to learning with true labels.

\subsection{Tiny ImageNet}
Tiny ImageNet represents 200 classes with 500 images in each class from ImageNet dataset \cite{ILSVRC15}. Images are cropped to $64\times64$. Detailed information is lost during the down-sampling process, making it more difficult to learn. ResNet-18 for ImageNet \cite{he2016deep} is deployed. Instead of using the original first convolutional layer with a $7\times7$ kernel and the subsequent max pooling layer, we replace them with a convolutional layer with a $3\times3$ kernel, stride=1, and no padding. The initial learning rate is 0.1, divided by 10 after 20,000 and 40,000 iterations. The batch size is 256 and weight decay is $5e-4$. Other settings are the same as CIFAR100. The experimental results are shown in the third column of Table \ref{tab_cifar100}. We also only test our method under the setting ``with0''. ``PC/S'' cannot converge here, but our method still achieves promising performance.

\subsection{Discussions}
In this section, we aim to verify the following facts about the proposed method: (1) In practice, our proposed method may not require Q to be invertible in some cases; (2) using randomly generated transition matrices and manually designed transition matrices achieves comparable performance; (3) the convergence rate of the proposed method can benefit more from the biased complementary labels than uniform complementary labels.

\textbf{Non-invertible $\mathbf{Q}$.}
In practice, our proposed method does not require $\mathbf{Q}$ to be invertible. To verify this, we test our method on MNIST dataset under the ``with0'' setting in which the complementary labels are generated according to a non-invertible transition matrix. This transition matrix is randomly generated and shown as follows:
\begin{equation} \label{mnist_true}
\tiny
\begin{bmatrix}
0 & .3042 & 0 & 0 & .5197 & 0 & 0 & .1762 & 0 & 0 \\
0 & 0 & 0 & .6141 & 0 & .0478 & 0 & .3381 & 0 & 0 \\
0 & 0 & 0 & 0 & 0 & .0213 & 0 & .7834 & .1953 & 0 \\
0 & .4489 & 0 & 0 & .1034 & 0 & 0 & .4477 & 0 & 0 \\
0 & .2596 & 0 & 0 & 0 & .3836 & 0 & 0 & .3568 & 0 \\
0 & .4416 & 0 & .5412 & 0 & 0 & 0 & 0 & .0172 & 0 \\
0 & .3049 & 0 & 0 & 0 & .5828 & 0 & .1123 & 0 & 0 \\
0 & .1499 & 0 & 0 & 0 & .3134 & 0 & 0 & 0 & .5367 \\
.1061 & .4718 & 0 & 0 & 0 & 0 & 0 & .4220 & 0 & 0 \\
0 & .4261 & .0579 & 0 & 0 & 0 & .5160 & 0 & 0 & 0
\end{bmatrix},
\end{equation}

 The results are shown in Table \ref{tab_mnist_noninvertible}. Compared with the results with respect to the invertible transion matrix, we can see the performance does not decrease, which indicates our proposed method can be applied to very general settings of complementary labels.

\begin{table} [t!]
\centering
\caption{The classification accuracy on MNIST with respect to the non-invertible transition matrix. ``TL'' denotes the result of learning with true labels. ``LM/T'' and ``LM/E'' refer to our method with the true $\mathbf{Q}$ and the estimated one, respectively. }
\label{tab_mnist_noninvertible}
\begin{tabular}{l|c|c|c}
\hline
Method & TL & LM/T & LM/E \\ \hline
non-invertible & 99.12 & $98.65\pm0.01$ & $98.60 \pm 0.01$ \\
invertible & 99.12 & $98.63\pm0.05$ & $98.61\pm0.05$ \\
\hline
\end{tabular}
\vspace{0pt}
\end{table}

 \textbf{Manually Designed $\mathbf{Q}$.}
In this experiment, rather than randomly selecting complementary labels for each class, we manually determine the transition matrix at first. The manually designed transition matrices in the ``without0'' and ``with0'' setting and their corresponding estimates are shown in Figure \ref{fig_without0} and \ref{fig_with0}, respectively. The classification accuracies are reported in Table \ref{tab_mnist_manualwithout0} and \ref{tab_mnist_manualwith0}, respectively. We can see that the performance with respect to randomly selected and manually designed transition matrices is similiar.

\begin{figure}[t!] 
\centering
\begin{subfigure}
{\includegraphics[width=0.47\columnwidth]{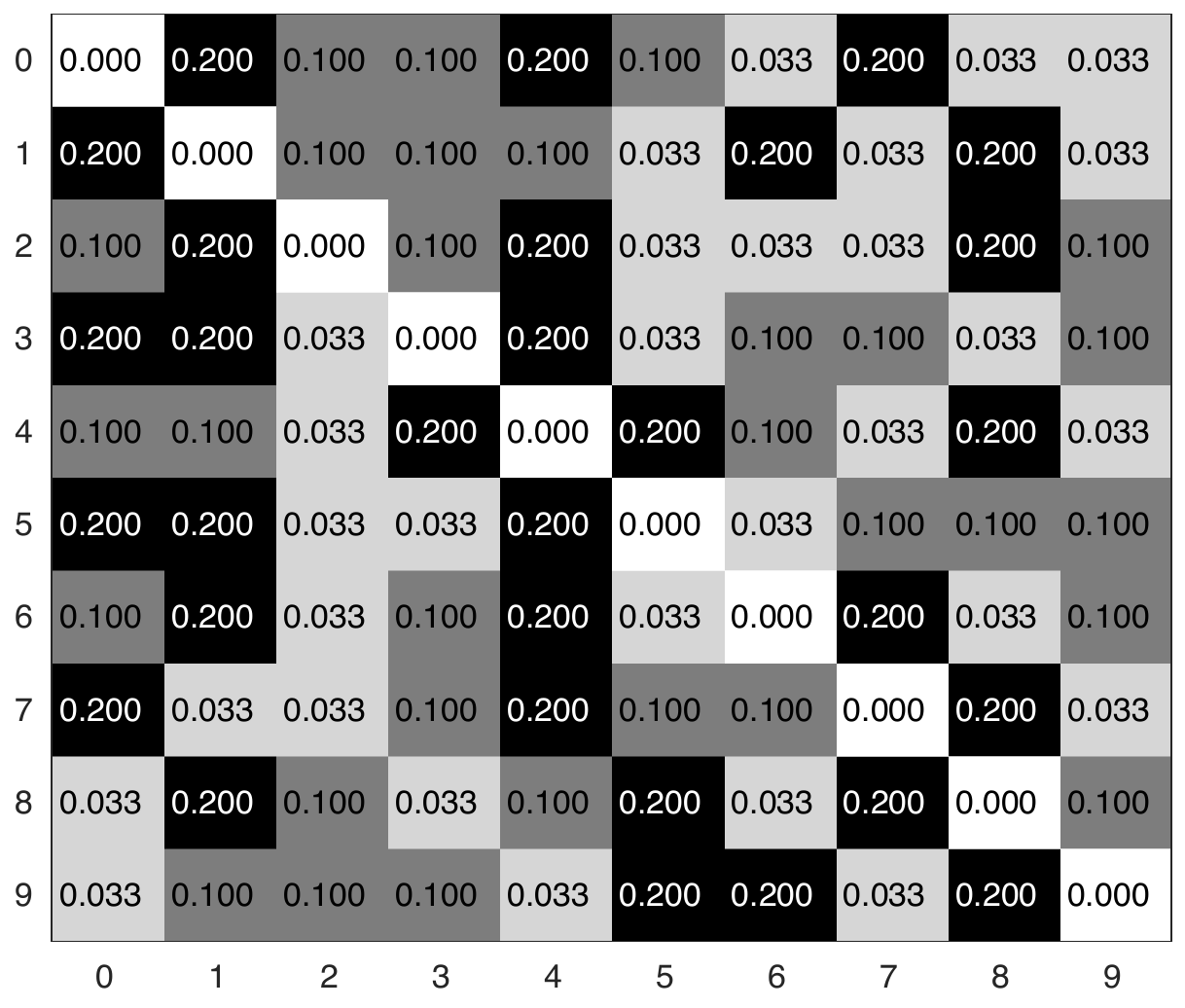}} 
\end{subfigure}
\begin{subfigure}
{\includegraphics[width=0.494\columnwidth]{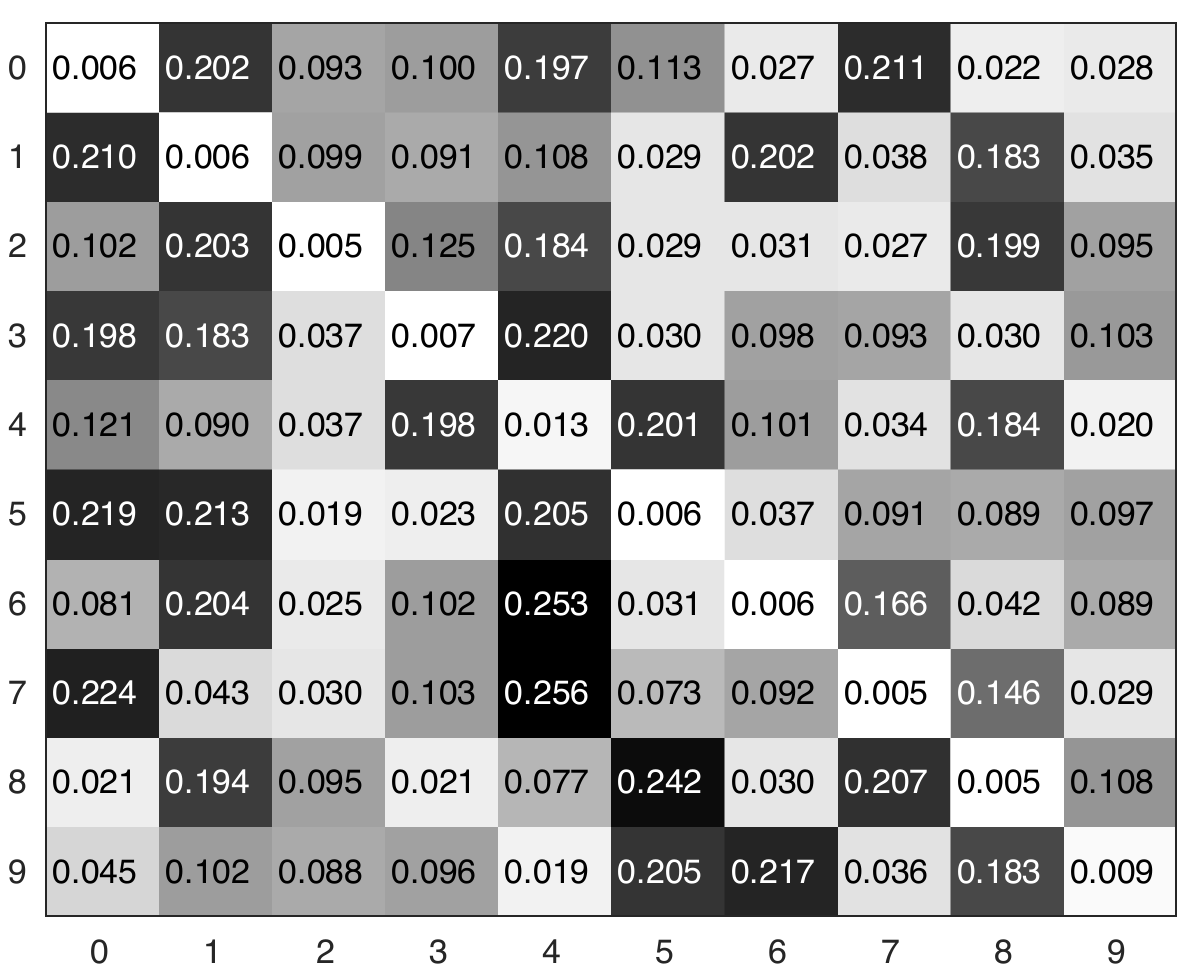}}
\caption{The true transition matrix and its estimate in the ``without0'' setting.}\label{fig_without0}
\end{subfigure}
\end{figure}

\begin{figure}[t!]
\centering
\begin{subfigure}
{\includegraphics[width=0.47\columnwidth]{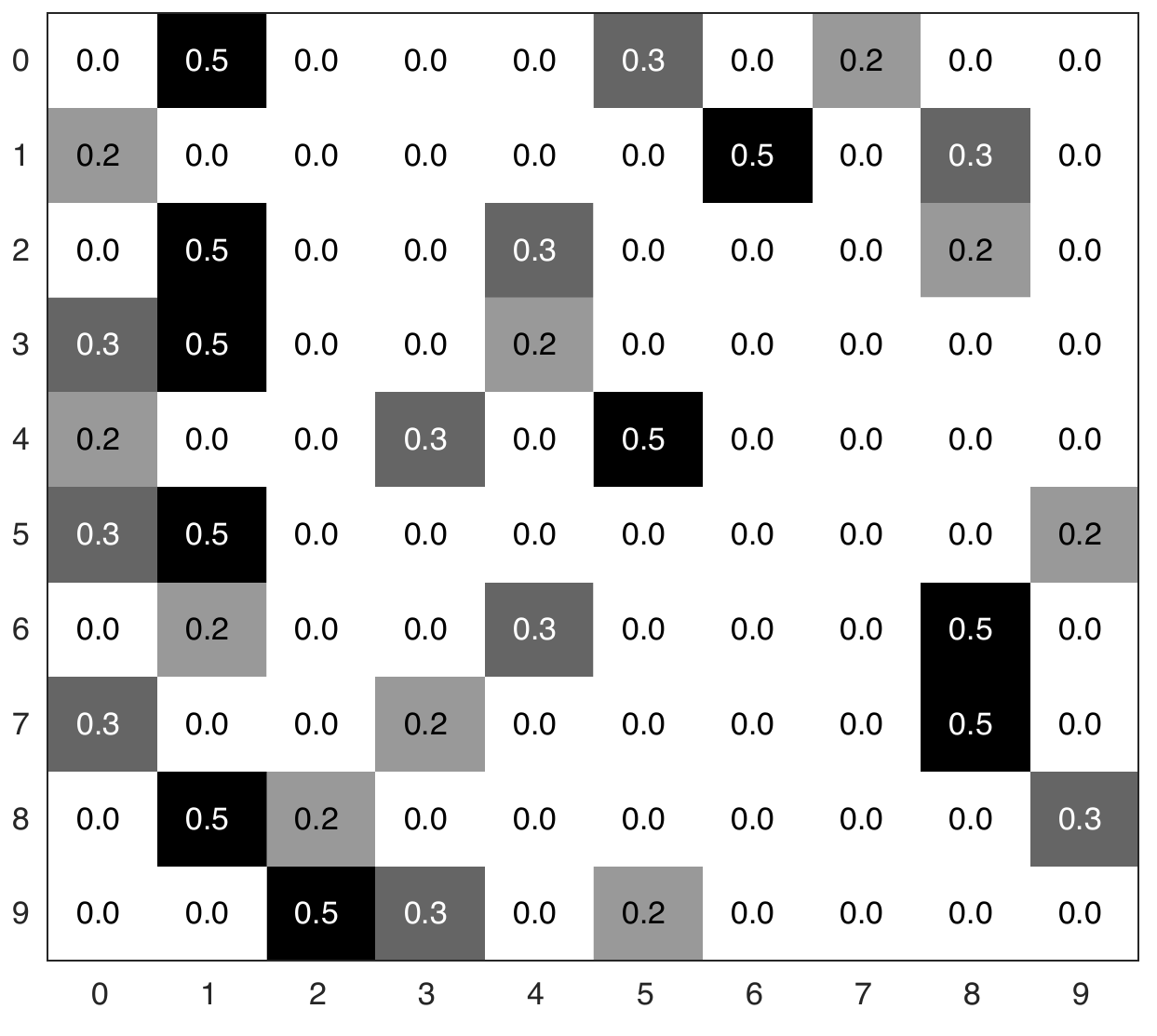}}
\end{subfigure}
\begin{subfigure}
{\includegraphics[width=0.494\columnwidth]{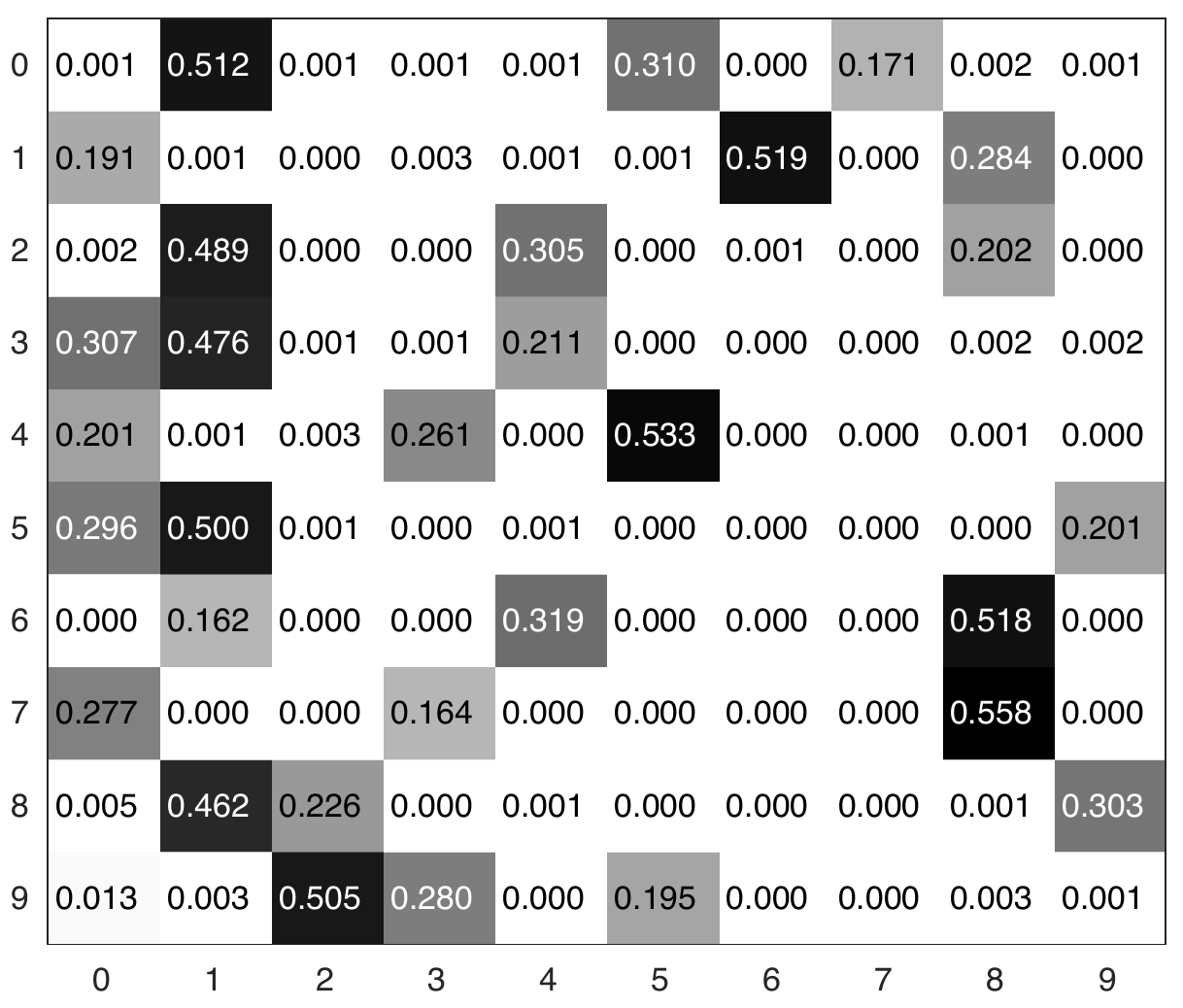}}
\caption{The true transition matrix and its estimate in the ``with0'' setting.}\label{fig_with0}
\end{subfigure}
\end{figure}

\begin{table}[t!]
\centering
\caption{The classification accuracy on MNIST with respect to the manually designed transition matrix in the ``without0'' setting. ``TL'' denotes the result of learning with true labels. ``LM/T'' and ``LM/E'' refer to our method with the true $\mathbf{Q}$ and the estimated one, respectively.}
\label{tab_mnist_manualwithout0}
\begin{tabular}{l|c|c|c}
\hline
Method & TL & LM/T & LM/E \\ \hline
manual & 99.12 & $97.15\pm0.06$ & $96.58 \pm 0.30$ \\
randomness & 99.12 & $97.65\pm0.15$ & $97.04\pm0.31$ \\
\hline
\end{tabular}
\vspace{0pt}
\end{table}

\begin{table}[t!]
\centering
\caption{The classification accuracy on MNIST with respect to the manually designed transition matrix in the ``with0'' setting. ``TL'' denotes the result of learning with true labels. ``LM/T'' and ``LM/E'' refer to our method with the true $\mathbf{Q}$ and the estimated one, respectively.}
\label{tab_mnist_manualwith0}
\begin{tabular}{l|c|c|c}
\hline
Method & TL & LM/T & LM/E \\ \hline
manual & 99.12  & $98.56\pm0.06$ & $98.53 \pm 0.04$ \\
randomness & 99.12 & $98.63\pm0.05$ & $98.61\pm0.05$ \\
\hline
\end{tabular}
\vspace{0pt}
\end{table}

\textbf{The Benefit of Biased Complementary Labels.}
\begin{figure}[t!]
\centering
{\includegraphics[width=0.8\columnwidth]{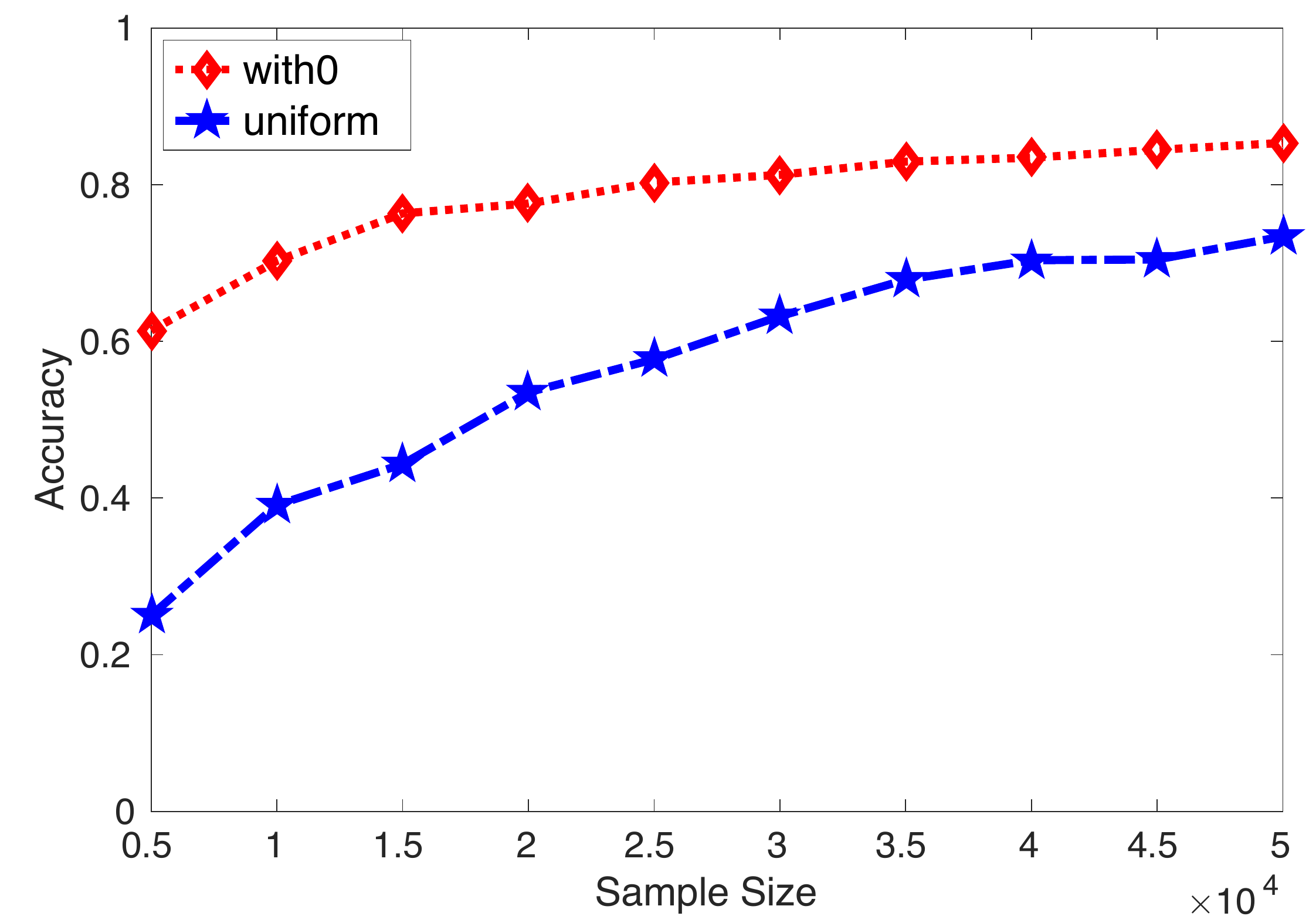}} 
\caption{The convergence analysis of the proposed method under the uniform and biased assumption. } \label{fig_convergence}
\vspace{-10pt}
\end{figure}
Here, we aim to show that the convergence rate of the proposed method can benefit from the biased setting. These experiments are conducted on CIFAR10 dataset. We do not employ the MNIST dataset because it is easy to learn a good classifier from this data and the proposed method can achieve very high performance in both uniform and biased settings. Here, we randomly select a fraction of data from the CIFAR10 dataset, increasing the total sample size of training data from 5,000 to 50,000 with step size 5,000. The classification accuracies with respect to these training data in both ``uniform'' and ``with0'' settings are reported. As shown in Figure \ref{fig_convergence}, the proposed method achieves high performance when sample sizes are small and also converges quickly. Our experimental results provide a good guidance of assigning complementary labels; that is, in practice, for all instances in a certain class, the complementary labels can be selected from a small subset of the label space and the remaining class labels in the label space have no chances to be selected as complementary labels for this class. For example, in the left subfigure of Figure \ref{fig_with0}, for all digits 0, we only assign the labels ``1'', ``5'', and ``7'' as complementary labels. In this way, we can learn an effective classifier with relatively small training dataset. The only requirement for selecting the biased complementary labels is that all class labels in the label space should be assigned as complementary labels. For example, seen in the left subfigure of Figure \ref{fig_with0}, digit ``6'' is only assigned as complementary labels of the digits ``1''. This is OK for learning a good classifier. However, if ``6'' is never assigned as a complementary label. We have no information for this class, which makes it impossible for the algorithm to learn an effective classifier.


\section{Conclusion}
We address the problem of learning with biased complementary labels. Specifically, we consider the setting that the transition probabilities $P(\bar{Y}=j|Y=i), \forall i \neq j$ vary and most of them are zeros. We devise an effective method to estimate the transition matrix given a small amount of data in the anchor set. Based on the transition matrix, we proposed to modify traditional loss functions such that learning with complementary labels can theoretically converge to the optimal classifier learned from examples with true labels. Comprehensive experiments on a wide range of datasets verify that the proposed method is superior to the current state-of-the-art methods. 

\textbf{Acknowledgement.} This work was supported by Australian Research Council Projects FL-170100117, DP-180103424, and LP-150100671. This work was partially supported by SAP SE and research grant from Pfizer titled ``Developing Statistical Method to Jointly Model Genotype and High Dimensional Imaging Endophenotype''. We are also grateful for the computational resources provided by Pittsburgh Super Computing grant number TG-ASC170024.

\newpage
\appendix
\section{Proof of Theorem 1} \label{appA}
\begin{proof}
According to Assumption 1 and based on the modified loss function, when learning from examples with complementary labels, we also have 
\[q_i^*(X) = P(\bar{Y}=i|X), \forall i \in [c].\]

Let $\mathbf{v}(X)=[P(Y=1|X),\cdots,P(Y=c|X)]$ and $\bar{\mathbf{v}}(X)=[P(\bar{Y}=1|X),\cdots,P(\bar{Y}=c|X)]$. According to Eq. (11), we have 
\begin{equation}
\bar{\mathbf{v}}(X) = \mathbf{Q}^\top \mathbf{v}(X),
\end{equation}
which further ensures
\begin{equation}
\mathbf{q}^*(X) = \mathbf{Q}^\top \mathbf{v}(X) = \mathbf{Q}^\top \mathbf{g}^*(X).
\end{equation}

If the transition matrix $\mathbf{Q}$ is invertible, then we find the optimal $\mathbf{g}^*(X)=\mathbf{v}(X)$ which ensures $\bar{f}^* = f^*$. The proof is completed.
\end{proof}

\section{Proof of Lemma 1} \label{appB}
\begin{proof}
According to the definition of the cross-entropy loss, the loss is nonnegative. Then minimizing $R(f)$ can be obtained by minimizing the conditional risk $\mathbb{E}_{P_{Y|\mathbf{X}=\mathbf{x}}}[\ell(f(\mathbf{x}),Y)|\mathbf{X}=\mathbf{x}]$ for every $\mathbf{x} \in \mathcal{X}$ \cite{masnadi2009design}, and the conditional risk is defined as
\begin{equation} \label{cond_risk_true}
\psi(\mathbf{g}) = -\sum_{i=1}^c P(Y=i|\mathbf{x})\log(g_i(\mathbf{x})),
\end{equation}
where $P(Y=i|\mathbf{x})$ is short for $P(Y=i|\mathbf{X}=\mathbf{x})$.

The problem turns to minimizing $\psi(\mathbf{g})$ subject to $\mathbf{g}(\mathbf{x})\in\Delta^{c-1}$, $\forall \mathbf{x} \in \mathcal{X}$. Then by using the Lagrange Multiplier method \cite{bertsekas1999nonlinear}, we have
\[\mathcal{L}=\psi(\mathbf{g}) -\lambda(\sum_{i=1}^c g_i(\mathbf{x})-1).\]

The derivative of $\mathcal{L}$ with respect to $\mathbf{g}$ is,
\[\frac{\partial \mathcal{L}}{\partial \mathbf{g}} = [-\frac{P(Y=1|\mathbf{x})}{g_1(\mathbf{x})}-\lambda, \cdots, -\frac{P(Y=c|\mathbf{x})}{g_c(\mathbf{x})}-\lambda]^\top,\]
which is equal to $\mathbf{0}$ when substituting $\mathbf{g}$ with the minimizer $\mathbf{g}^*$, i.e., $\frac{\partial \mathcal{L}}{\partial \mathbf{g}^*} = \mathbf{0}$. Then
\[g^*_i(\mathbf{x}) = -\lambda P(Y=i|\mathbf{x}), \forall i\in [c] \textrm{ and } \forall \mathbf{x}\in \mathcal{X}.\]

Because $\sum_{i}^c g^*_i(\mathbf{x})=1$ and $\sum_{i=1}^c P(Y=i|\mathbf{x})=1$, then we have
\[\sum_{i=1}^c g^*_i(\mathbf{x}) = -\lambda \sum_{i=1}^c P(Y=i|\mathbf{x}) = 1.\]

Thus we easily get $\lambda=-1$ and $g_i^*(\mathbf{x})=P(Y=i|\mathbf{x}), \forall i \in [c]$ and $\forall \mathbf{x} \in \mathcal{X}$. The proof is completed.
\end{proof}

\section{Proof of Lemma 2} \label{appC}
In order to prove Lemma 2, we need the loss function $\bar{\ell}(f(X),\bar{Y}=i)$ to be Lipschitz continous with respect to $h_i(X)$, which can be proved by the following lemma,
\begin{mylemma} \label{lipschtiz}
Suppose that all column of the transition matrix is not all-zero, then loss function $\bar{\ell}(f(X),\bar{Y}=i)$ is 1-Lipschitz with respect to $h_j(X), \forall j \in [c]$.
\end{mylemma}
\begin{proof}
Recall that 
\begin{equation}
\bar{\ell}(f(X),\bar{Y}=i) = -\log\left(\frac{\sum_{k=1}^c Q_{ki}\exp(h_k(X))}{\sum_{k=1}^c \exp(h_k(X))}\right).
\end{equation}
Take the derivative of $\bar{\ell}(f(X),\bar{Y}=i)$ with respect to $h_j(X)$, we have 
\begin{equation} \label{derivative}
\begin{aligned}
&\frac{\partial \bar{\ell}(f(X),\bar{Y}=i)}{\partial h_j(X)} \\
&= -\frac{Q_{ji}\exp(h_j(X))}{\sum_{k=1}^c Q_{ki}\exp(h_k(X))} + \frac{\exp(h_j(X))}{\sum_{k=1}^c \exp(h_k(X))}.
\end{aligned}
\end{equation}

According to Eq.(\ref{derivative}), it is easy to conclude that $-1 \leq \frac{\partial \bar{\ell}(f(X),\bar{Y}=i)}{\partial h_j(X)} \leq 1$, which also indicates that the loss function is 1-Lipschitz with respect to $h_j(X), \forall j \in [c]$. The proof is completed.
\end{proof}

Now we are ready to prove Lemma 2.
\begin{proof}
Since the softmax function preserve the rank of its inputs, $f(X)=\arg\max_{i \in [c]} g_i(X) = \arg\max_{i \in [c]} h_i(X)$. We thus have
\begin{equation}
\begin{aligned}
&\mathfrak{R}_{n_{i}}(\bar{\ell}\circ\mathcal{F})\\
&=\mathbb{E}\left[\sup_{f\in \mathcal{F}}\frac{1}{n_i}\sum_{j=1}^{n_i}\sigma_j\bar{\ell}(f(X_j),\bar{Y}_j=i)\right] \\
&= \mathbb{E}\left[\sup_{\arg\max\{h_1(X),\cdots,h_c(X)\} }\frac{1}{n_i}\sum_{j=1}^{n_i}\sigma_j\bar{\ell}(f(X_j),\bar{Y}_j=i)\right] \\
&= \mathbb{E}\left[\sup_{\max\{h_1(X),\cdots,h_c(X)\} }\frac{1}{n_i}\sum_{j=1}^{n_i}\sigma_j\bar{\ell}(f(X_j),\bar{Y}_j=i)\right] \\
&\leq \mathbb{E}\left[ \sum_{k=1}^c \sup_{h_k(X) }\frac{1}{n_i}\sum_{j=1}^{n_i}\sigma_j\bar{\ell}(f(X_j),\bar{Y}_j=i)\right] \\
&= \mathbb{E}\left[ \sum_{k=1}^c \sup_{h_k(X) }\frac{1}{n_i}\sum_{j=1}^{n_i}\sigma_j\log\left(\frac{\sum_{m=1}^c Q_{mi} \exp(h_m(X))}{\sum_{m=1}^c \exp(h_m(X))} \right)\right] \\ 
&= \sum_{k=1}^c \mathbb{E}\left[ \sup_{h_k(X) }\frac{1}{n_i}\sum_{j=1}^{n_i}\sigma_j\log\left(\frac{\sum_{m=1}^c Q_{mi} \exp(h_m(X))}{\sum_{m=1}^c \exp(h_m(X))} \right)\right].
\end{aligned}
\end{equation}
Here, the argument $f\in\mathcal{F}$ of $\sup$ function indicates that $f$ is chosen from the function space $\mathcal{F}$. The function space $\mathcal{F}$ is actually determined by the function space of $\mathbf{h}$ due to the fact that $f = \arg\max\{g_1(X),\cdots,g_c(X)\}=\arg\max \{h_1(X),\cdots,h_c(X)\}$. Thus, the argument of $\sup$ function can be changed to $\arg\max{h_1(X),\cdots,h_c(X)}$ in the second equality. Since $\arg\max \{h_1(X),\cdots,h_c(X)\}$ and $\max\{h_1(X),\cdots,h_c(X)\}$ give the same constraint on $h_i(X), \forall i \in[c]$, the argument is changed to $\max\{h_1(X),\cdots,h_c(X)\}$ in the third equality.

According to Lemma \ref{lipschtiz} and Talagrand’s contraction lemma \cite{ledoux2013probability}, we have
\begin{equation}
\begin{aligned}
\mathfrak{R}_{n_{i}}(\bar{\ell}\circ\mathcal{F}) &\leq \sum_{k=1}^c \mathbb{E}\left[ \sup_{h_k(X) }\frac{1}{n_i}\sum_{j=1}^{n_i}\sigma_j h_k(X)\right]\\
&= \sum_{k=1}^c \mathfrak{R}_{n_{i}}(\mathcal{H}) \\
&= c\mathfrak{R}_{n_{i}}(\mathcal{H}),
\end{aligned}
\end{equation}
The proof is completed.
\end{proof}

\newpage
\bibliography{egbib}
\end{document}